\documentclass{ifacconf}

\usepackage{amsmath}
\usepackage{amssymb}
\usepackage{comment}
\usepackage{todonotes}
\usepackage{color}
\usepackage{url}
\usepackage{amsfonts}
\usepackage{bm}
\usepackage{csquotes}
\usepackage{enumitem}
\usepackage{xcolor}
\usepackage{tikz}
\usepackage{graphicx}      
\usepackage{natbib}        
\usepackage{balance}

\newenvironment{proof}{{\it Proof:}}{$\triangleleft$ }

\newtheorem{remark}{Remark}[section]
\newtheorem{definition}{Definition}[section]

\newtheorem{proposition}{Proposition}[section]

\newtheorem{problem}{Problem}[section]

\newtheorem{example}{Example}
\newenvironment{example*}
  {\par\noindent\textit{Example.}\ \itshape}
  {\par}

\begin{document}
\begin{frontmatter}

\title{On the Stabilization of Rigid Formations on Regular Curves} 

\author[First]{Mohamed Elobaid} 
\author[First]{Shinkyu Park}
\author[First]{Eric Feron} 

\address[First]{Computer, Electrical and Mathematical Sciences and Engineering \emph{CEMSE}, King Abdullah University of Science and Technology \emph{(KAUST)}; Thuwal 23955-6900, Saudi Arabia (e-mail: \{mohamed.elobaid,shinkyu.park,eric.feron\}@kaust.edu.sa).} 

\begin{abstract}                
This work deals with the problem of stabilizing a multi-agent rigid formation on a general class of planar curves. Namely, we seek to stabilize an equilateral polygonal formation on closed planar differentiable curves after a path sweep. The task of finding an inscribed regular polygon centered at the point of interest is solved via a randomized multi-start Newton-Like algorithm for which one is able to ascertain the existence of a minimizer. Then we design a continuous feedback law that guarantees convergence to, and sufficient sweeping of the curve, followed by convergence to the desired formation vertices while ensuring inter-agent avoidance. The proposed approach is validated through numerical simulations for different classes of curves and different rigid formations. Code: \url{https://github.com/mebbaid/paper-elobaid-ifacwc-2026} 

\end{abstract}


\end{frontmatter}


\section{Introduction}
Coordinating multiple autonomous agents to perform complex spatial tasks—such as search and rescue, environmental monitoring, or cooperative inspection—poses several challenges. In such scenarios, agents are required to execute persistent coverage patterns over large areas following prescribed curves (e.g., Lissajous or Spirographs \cite{lissajous}). Once a region of interest is identified, the control objective transitions from wide coverage to focused inspection, where agents must converge to a rigid formation centered on the detected target. A regular polygonal formation is particularly desirable in this context, as it offers wide coverage and communication symmetry for a given number of agents.

Achieving this mission profile requires a control architecture capable of unifying distinct objectives: $(i)$ a path-following objective that guarantees convergence to a prescribed, possibly non-simple, geometric curve; and $(ii)$ a formation-stabilization component that achieves convergence to a rigid configuration \textit{on} that curve. Both objectives are required while ensuring inter-agent safety.

Existing frameworks typically treat these problems separately. Concerning $(i)$, several path-following approaches are available in the literature, such as energy-shaping techniques \cite{Ortega}, sliding-mode controllers \cite{Ozguner}, and transverse feedback linearization (TFL) \cite{Hauser, Altafini, Manf1, ElobaidLCSS}. However, standard path following via TFL fail on the very class of non-simple curves (i.e., those with cusps or self-intersections) often used for coverage \cite{pathfollowing}.
Regarding $(ii)$, formation coordination have been addressed through hierarchical designs \cite{doosthoseini2015coordinated, CoordinatedUAVs}, where agents are driven first to desired paths and subsequently synchronized. These works, however, typically focus on coordinating the dynamic motion of the agents along the paths, rather than our objective of a smooth transition to a static terminal formation given decoupling matrix singularities. Other recent works on formation control \cite{mattia} have explored formation control, but not in this specific multi-objective hierarchical context.

In this work, we propose a unified control architecture that solves this multi-objective problem. First, to address $(i)$ for non-simple curves, we introduce a \emph{lifted representation} of the curve, borrowing from \cite{lifting-vectorfield}. By augmenting each agent's state with a virtual lifted coordinate $z_i$, the path following problem is treated with respect to a higher-dimensional submanifold which is, by construction, simple and free of intersections. On this manifold, a TFL controller is derived to stabilize the transverse coordinates and guarantee convergence to the path.
Concerning $(ii)$, we cast the geometric problem of finding the terminal formation as a nonlinear least-squares program. We propose a randomized multi-start Newton-like solver that ensures the existence of a minimizer for broad classes of closed $\mathbf{C}^1$ curves, providing the desired vertex locations.

Finally, to unify these components, we design a continuous, state-dependent blending map. This map guarantees a smooth transition from the TFL path-following controller to a local pose stabilization controller, achieving convergence to the assigned vertices without switching-induced discontinuities. A distributed, asymmetric collision avoidance scheme is incorporated to preserve inter-agent safety throughout all phases of motion. 
The paper is organized as follows: Section II details the admissible class of curves, defines rigid formations and formally state the problem. Section III presents the main results. Simulations are discussed in Section IV. Concluding remarks end the manuscript.


\section{Assumptions and Problem Statement}

\subsection{The Class of Admissible Curves}
We begin with a few definitions and comments;

\begin{definition}\label{def:path}
A parameterized curve is a mapping 
$\gamma : I \to \mathbb{R}^2, \, \,  s \mapsto \gamma(s),
$
where $I \subseteq \mathbb{R}$ is an interval (\cite{carmo}). 
$\gamma$ is said to be closed if for a period $L>0$
\[
\gamma(0) = \gamma(L), \quad \gamma'(0) = \gamma'(L), \quad \gamma''(0) = \gamma''(L), \ \dots
\]
Unless otherwise stated, $\gamma$ is at least continuously differentiable (of class $\mathbf{C}^1$) with respect to $s$. Its image,
\begin{align}\label{curve}
    \Gamma = \gamma(I) \subset \mathbb{R}^2,
\end{align}
is referred to as the \emph{geometric path}.
\end{definition}

\begin{definition}
 $\gamma(s)$ is said to be \emph{regular} at $s_\circ\in I$ if its tangent vector is well defined and nonzero $\gamma'(s_\circ) \not = \mathbf{0}$.  
A point $s_\circ \in I$ where $\gamma'(s_\circ) = \mathbf{0}$ or is undefined is called a \emph{singularity}.  
A common type of singularity is a \emph{cusp}, where the curve reverses its direction of travel in the plane. Cusps can occur on $\mathbf{C}^1$ curves.
\end{definition}

\begin{definition}
$\gamma(s)$ is an \emph{immersion} if it is regular everywhere on $I$.  
If, in addition, $\gamma$ is injective (i.e., $\gamma(s_1) \neq \gamma(s_2)$ for $s_1 \neq s_2$ in the interior of $I$), it is called an \emph{embedding}.  
The image of an embedding is a simple, non-self-intersecting curve in $\mathbb{R}^2$.
\end{definition}

Singularities, in our sense, may arise from two distinct mechanisms: 
\emph{(i)} a non-differentiable parametrization exhibiting a discontinuity in its tangent vector, or 
\emph{(ii)} a $\mathbf{C}^1$ parametrization whose tangent vector continuously vanishes at isolated points. 
Only parametrizations of the second type are \emph{admissible} in our setting. Note that it is possible, in some cases, to obtain a differentiable reparametrization by letting the velocity of the parameter vary along the curve. 
\begin{example}
Consider the simple example of a perfect square on the plane. In the standard arc-length parametrization, each edge is traversed at constant speed. 
        At the vertices, the tangent vector undergoes a jump (e.g., from $(1,0)$ to $(0,1)$), so that $\gamma'(s)$ is discontinuous.  Alternatively, a $\mathbf{C}^1$ parametrization can be constructed by enforcing a vanishing velocity at each corner through a cubic Hermite interpolation along each edge:
        \[
          t(s) = L\bigl(3(s/L)^2 - 2(s/L)^3\bigr), \qquad s \in [0,L].
        \]
        ensuring continuity of $\gamma'(s)$ and $\gamma'(s_\mathrm{corner}) = \mathbf{0}$ at the vertices. 
\end{example}        
We will treat any geometric path $\Gamma(I)$ that is the image of a $\mathbf{C}^1$ parametrization of a curve with isolated points of vanishing tangent. We explicitly do \emph{not} preclude curves with cusps nor self-intersections, provided they are differentiable. 
The machinery we will utilize, as the reader may have already inferred, is a \emph{lifting construction} which we make explicit in Section \ref{sec:controller-design}.

\begin{remark}
The restriction to closed curves is mainly motivated by the
formation–finding and sweeping tasks described in the introduction.
Nevertheless, the lifting and control constructions could be extended to open, bounded curves~$\gamma=[s_0,s_f]\mapsto\mathbb{R}^2$
provided some adjustments to the control law and rigid formation are carried.
\end{remark}

\subsection{The Class of Rigid Formations}

We now define what we mean by a rigid formation constrained to a geometric path.

\begin{definition}
Let $\gamma : I \to \mathbb{R}^2$ be an admissible curve and let $n \in \mathbb{N}$, $n \ge 3$, be the number of agents. A \emph{rigid formation} of $n$ agents is an ordered set of $n$ points, $\{p_0, p_1, \ldots, p_{n-1}\}$, located on the image of $\gamma$ that form the vertices of a \emph{regular polygon}.
\end{definition}

To find such a formation, we search for parameters $\theta = (\theta_0, \ldots, \theta_{n-1})^\top \in I^n$ whose corresponding points $\gamma(\theta_i)$ satisfy the geometric properties of a regular polygon, and whose center is as close as possible to a point of interest $\mathbf{c}_{\text{target}}$. For compactness, let $\gamma_i := \gamma(\theta_i)$ be the $i$-th vertex. We then define the \emph{edge vector} $\mathbf{e}_i$ as the vector pointing from vertex $i-1$ to vertex $i$: $\mathbf{e}_i(\theta) = \gamma_i - \gamma_{i-1}.
$
Using this shorthand notation, we have:

\begin{definition}
\label{def:rigid-formation-residuals}
The configuration $P(\theta^\star) = [\gamma_0, \ldots, \gamma_{n-1}]$ forms a regular polygon if the following residuals vanish for all $i$ (modulo $n$):
\begin{enumerate}
    \item \emph{Equal Side Lengths residual:}
    \begin{equation}\label{eq:residual-length}
    r_{L,i}(\theta^\star) = \|\mathbf{e}_{i+1}\|^2 - \|\mathbf{e}_i\|^2 = 0.
    \end{equation}

    \item \emph{Equal Interior Angles residual:} The dot product between consecutive edge vectors must be constant:
    \begin{equation}\label{eq:residual-angle}
    r_{A,i}(\theta^\star)
    = \mathbf{e}_{i+1}^\top \mathbf{e}_i
    - \mathbf{e}_{i+2}^\top \mathbf{e}_{i+1}
    = 0.
    \end{equation}
\end{enumerate}
\end{definition}

The task of finding such a configuration is posed as an optimization problem. We define the \emph{rigid formation cost functional} $J: I^n \to \mathbb{R}$ as the weighted sum:
\begin{equation}
J(\theta)
= \frac{1}{2} \sum_{i=0}^{n-1}
\left( w_L\, r_{L,i}(\theta)^2 + w_A\, r_{A,i}(\theta)^2 \right),
\label{eq:formation-cost}
\end{equation}
where $w_L > 0$ and $w_A > 0$ are weighting factors. A configuration $\theta^\star$ corresponds to a rigid formation if it is a global minimizer and satisfies $J(\theta^\star) = 0$.

\subsection{System Model and Problem Statement}
Consider a system of $n$ agents, each modeled by an \emph{extended unicycle system} obtained via dynamic extension of the standard kinematics, ensuring well-defined relative degrees with respect to position outputs~\cite[Ch.~4]{isidori}: 
\begin{equation}\label{eq:agent-model}
\Sigma_i: \quad
\begin{cases}
    \dot{\mathbf{p}}_i = v_i \begin{bmatrix} \cos\psi_i \\ \sin\psi_i \end{bmatrix} \\
    \dot{\psi}_i = \omega_i \\
    \dot{v}_i = a_i
\end{cases}
\end{equation}
where $\mathbf{p}_i = [x_i, y_i]^\top$ is the agent's position in the plane, $\psi_i$ its heading, and $v_i$ its speed. The controls are the longitudinal acceleration $a_i$ and the angular velocity $\omega_i$. Let $\mathbf{p}(t) = [\mathbf{p}_0(t)^\top, \ldots, \mathbf{p}_{n-1}(t)^\top]^\top$ be the vector of all agent positions. The problem we are interested in asks the following;

\begin{problem}
\label{problem}
Given an admissible curve $\gamma(s)$ and a desired rigid formation defined by target vertices $P^\star = \{p_0^\star, \ldots, p_{n-1}^\star\} \subset \Gamma$, find, if possible, a smooth feedback law $\mathbf{u}_i = [a_i, \omega_i]^\top = f_i(\mathbf{p}_i, \psi_i, v_i, \{\mathbf{p}_j\}_{j\neq i})$ for each agent $i$ such that for a set of initial conditions $\mathcal{X} \supset \Gamma$, the closed-loop multi-agent system satisfies:
\begin{enumerate}
    \item \emph{Path following:} The geometric path $\Gamma$ is an attractive and invariant set for each agent (up to collision-avoidance). That is, the distance of each agent to the curve, $\|\mathbf{p}_i(t)\|_\Gamma = \inf_{\mathbf{q} \in \Gamma} \|\mathbf{p}_i(t) - \mathbf{q}\|$, converges to zero:
    \[
    \lim_{t \to \infty} \|\mathbf{p}_i(t)\|_\Gamma = 0, \quad \forall i \in \{0, \ldots, n-1\}.
    \]
    
    \item \emph{Rigid Formation Stabilization:} The positions of the agents converge to the set of target vertices, after satisfying the sweeping requirement:
    \[
    \lim_{t \to \infty} (\mathbf{p}_i(t) - p_i^\star) = \mathbf{0}, \quad \forall i \in \{0, \ldots, n-1\}.
    \]

    \item \emph{Inter-Agent Collision Avoidance:} The agents maintain a minimum separation $d_{\text{safe}} > 0$ for all time:
    \[
    \|\mathbf{p}_i(t) - \mathbf{p}_j(t)\| \ge d_{\text{safe}}, \quad \forall t \ge 0, \quad \forall i \neq j.
    \]
\end{enumerate}
\end{problem}


\section{Main Results}
\subsection{The Rigid Formation Finder}
Consider the minimizer
\[
\theta^\star = \arg\min_{\theta \in [0,2\pi)^n} J(\theta),
\]
where $J(\theta)$ is as in~\eqref{eq:formation-cost}. Since each residual term only depends on a few neighboring parameters, the Jacobian is sparse. Explicitly, from \eqref{eq:residual-length}-\eqref{eq:residual-angle} one has;
\begin{equation}\label{eq:jacobian-length-correct}
\frac{\partial r_{L,i}}{\partial \theta_j} =
\begin{cases}
2\,\mathbf e_i^\top \gamma'_{i-1} & \text{if } j = i-1, \\
-2(\mathbf e_{i+1} + \mathbf e_i)^\top \gamma'_i & \text{if } j = i, \\
2\,\mathbf e_{i+1}^\top \gamma'_{i+1} & \text{if } j = i+1, \\
0 & \text{otherwise},
\end{cases}
\end{equation}
Likewise,
\begin{equation}\label{eq:jacobian-angle-correct}
\frac{\partial r_{A,i}}{\partial \theta_j} =
\begin{cases}
-\mathbf e_{i+1}^\top \gamma'_{i-1} & \text{if } j = i-1, \\
(\mathbf e_{i+1} - \mathbf e_{i} + \mathbf e_{i+2})^\top \gamma'_i & \text{if } j = i, \\
(\mathbf e_i + \mathbf e_{i+1} - \mathbf e_{i+2})^\top \gamma'_{i+1} & \text{if } j = i+1, \\
-\,\mathbf e_{i+1}^\top \gamma'_{i+2} & \text{if } j = i+2, \\
0 & \text{otherwise}.
\end{cases}
\end{equation}
These expressions can be computed efficiently. The gradient is then $\nabla J(\theta) = \nabla\mathbf{r}(\theta)^\top \mathbf{W} \mathbf{r}(\theta)$, where $\nabla\mathbf{r}(\theta)$ is the full Jacobian matrix assembled from the partial derivatives above and $\mathbf{W}$ is a diagonal matrix of weights.

\begin{proposition}
For the admissible curve $\gamma\in C^1$ the rigid formation cost $J$ attains at least one global minimizer on the compact set $[0,2\pi]^n$. Moreover, if a configuration $\theta^\star$ satisfies $\mathbf r(\theta^\star)=0$ and the Jacobian $D\mathbf r(\theta^\star)$ has full column rank, then $\theta^\star$ corresponds to an isolated rigid formation.
\end{proposition}
\begin{proof}
Each $r_i(\theta)$ is continuous by admissibility of $\gamma(s)$, hence the cost 
$J(\theta)=\tfrac12\,\mathbf r(\theta)^\top \mathbf W \mathbf r(\theta)$ is continuous on $[0,2\pi]^n$.
Because this domain is compact (serving as covering of the periodic $\theta$ space), the existence of a global minimizer follows from the Extreme Value Theorem \cite{calculus}. Now suppose $\theta^\star$ satisfies $\mathbf r(\theta^\star)=0$ and that the Jacobian 
$D\mathbf r(\theta^\star)$ has full column rank.  
Then $\mathbf r$ is locally injective around $\theta^\star$, and by the Implicit Function Theorem, there exists an open neighborhood 
$U\subset[0,2\pi]^n$ of $\theta^\star$ such that 
$\mathbf r(\theta)=0$ implies $\theta=\theta^\star$ for all $\theta\in U$.  
Hence $\theta^\star$ is an isolated zero of $\mathbf r$.
Since $J(\theta^\star)= 0$, and is nonnegative everywhere, $\theta^\star$ is a (global) minimizer of $J$ in $U$, corresponding to an isolated rigid formation.
\end{proof}

\subsubsection{Global Minimizer via a Multi-Start Gauss–Newton Strategy:}
\label{subsec:multistart}

The cost function~\eqref{eq:formation-cost} is continuously differentiable but generally nonconvex, 
especially for curves with nonuniform curvature or lacking geometric symmetry. 
Consequently, different initializations of the Gauss–Newton iteration may converge to distinct stationary points corresponding to locally optimal agent spacings along~$\gamma$. 
To enhance robustness and improve exploration of the search space, we employ a \emph{multi-start Gauss–Newton strategy}. 
At iteration $k$, the update is the standard step;
\begin{equation}
\theta^{k+1}
= \theta^k
- \eta_k \big(D\mathbf r(\theta^k)^\top \mathbf W D\mathbf r(\theta^k)\big)^{-1}
  D\mathbf r(\theta^k)^\top \mathbf W\mathbf r(\theta^k),
\label{eq:GNupdate}
\end{equation}
where $\eta_k \in (0,1]$ is a step size selected via an Armijo-type line search~\cite{armijo}.
The Jacobian $D\mathbf r(\theta^k)$ inherits a block-tridiagonal sparsity structure 
from~\eqref{eq:residual-length}–\eqref{eq:residual-angle}, allowing efficient evaluation of~\eqref{eq:GNupdate}.  

Since $J \in \mathbf{C}^1$ on the compact domain $[0,2\pi]^n$, and $D\mathbf r(\theta)$ is locally Lipschitz,  
every sequence $\{\theta^k\}$ generated by~\eqref{eq:GNupdate} admits accumulation points that are stationary.
Global optimality, however, requires additional structural assumptions on~$\gamma$. For highly symmetric curves (e.g., Cassini ovals, rose curves, or deltoids ...etc), intuition suggests that 
uniformly spaced initial parameters
\[
\theta^{(0)}_i = \tfrac{2\pi i}{n}, \quad i=0,\ldots,n-1,
\]
often suffice: the symmetry of $\gamma$ leads to a global minimizer with equally spaced points along the curve.

In contrast, when $\gamma$ possesses no nontrivial symmetries and the parameter speed  
$s \mapsto \|\gamma'(s)\|$ is nonconstant, 
equal spacing in $\theta$ often leads to clustering along the image $\Gamma = \operatorname{Im}(\gamma)$ creating a poor initial geometry that often leads to local minima.  
To mitigate this, Algorithm~\ref{alg:multistart} employs two complementary initialization strategies:

\begin{enumerate}
    \item \emph{Curvature-weighted:}  
    Choose $\theta_i^{(0)}$ satisfying
    \[
    \int_0^{\theta_i^{(0)}} \|\gamma'(s)\|\,ds = \frac{i}{n}L,
    \qquad L = \int_0^{2\pi} \|\gamma'(s)\|\,ds,
    \]
    thereby ensuring equal arclength spacing of the corresponding points $\gamma(\theta_i^{(0)})$ along the curve.

    \item \emph{Randomized:}  
    Draw $\theta_i^{(0)} \sim \mathrm{Unif}(0,2\pi)$,  
    order them to preserve agent sequence,  
    and perform multiple independent restarts.
\end{enumerate}

For curves with pronounced curvature variations, the curvature-weighted rule provides a geometry-aware initialization that accelerates convergence.  
Randomized sampling, in turn, promotes exploration of multiple attraction basins of~$J$, increasing the probability of capturing the global minimizer as $N_\mathrm{init}\!\to\!\infty$.  
Since $J$ is continuous on a compact domain, 
the best configuration obtained among a finite number of restarts approximates the global minimum arbitrarily closely as the initialization density increases.

\begin{figure}[h]
\centering
\fbox{
\begin{minipage}{0.95\linewidth}

\textbf{Algorithm: Rigid Formation Finder (Multi-Start Gauss--Newton)}

\smallskip
\textbf{Input:} Curve $\gamma(s)$, agents $n$, initializations $N_{\mathrm{init}}$, optional target center $\mathbf c_{\mathrm{target}}$  
\textbf{Output:} Selected rigid formation parameters $\theta^\star$

\begin{enumerate}[leftmargin=*]

\item Initialize  
\[
\mathcal{S} \gets \emptyset.
\]

\item For each initialization $m = 1,\dots, N_{\mathrm{init}}$:
    \begin{enumerate}
        \item Initialize  
        \[
        \theta_m^{(0)} \ \text{(curvature-weighted or random)}.
        \]

        \item Perform Gauss--Newton iterations ($k = 0,\dots,K_{\max}$):
        \begin{enumerate}
            \item Compute residuals and Jacobian:
            \[
            \mathbf r(\theta_m^{(k)}), \quad
            D\mathbf r(\theta_m^{(k)}).
            \]

            \item Update parameters:
            \begin{align*}
            \Delta\theta &= -\big(D\mathbf r^\top \mathbf W D\mathbf r\big)^{-1}
                            D\mathbf r^\top \mathbf W\,\mathbf r, \\
            \theta_m^{(k+1)} &= \theta_m^{(k)} + \eta_k \Delta\theta .
            \end{align*}

            \item Stop if $\|\Delta\theta\| < \varepsilon$.
        \end{enumerate}

        \item If $J(\theta_m^{(k)}) < \varepsilon_J$ and $\theta_m^{(k)}$ is feasible:
        \begin{enumerate}
            \item Compute center and average side length:
            \begin{align*}
            \mathbf c_m &= \frac{1}{n}\sum_{i=0}^{n-1}
                           \gamma(\theta_{m,i}), \\
            \ell_m &= \frac{1}{n}\sum_{i=0}^{n-1}
                     \|\gamma(\theta_{m,i+1}) - \gamma(\theta_{m,i})\|.
            \end{align*}

            \item Add $(\theta_m^{(k)}, \mathbf c_m, \ell_m)$ to $\mathcal{S}$.
        \end{enumerate}

    \end{enumerate}

\item Selection step:

If $\mathbf c_{\mathrm{target}}$ is given,
\[
\theta^\star =
\arg\min_{(\theta,\mathbf c,\ell)\in\mathcal{S}}
\|\mathbf c - \mathbf c_{\mathrm{target}}\|.
\]

Otherwise choose the formation with maximal average side length:
\[
\theta^\star =
\arg\max_{(\theta,\mathbf c,\ell)\in\mathcal{S}} \ell.
\]

\end{enumerate}

\end{minipage}
}
\caption{Rigid formation search using a multi-start Gauss--Newton method.}
\label{alg:multistart}
\end{figure}

Algorithm~\ref{alg:multistart} thus selects the lowest-cost stationary configuration among all runs.  
Step sizes $\eta_k$ are adapted to guarantee monotone decrease of $J$, 
and convergence is declared when both $\|\Delta\theta^k\|$ and the relative cost variation fall below prescribed tolerances.  
Thanks to the sparsity of $D\mathbf r(\theta)$, the computational cost scales linearly with $n$, rendering the scheme suitable for real-time formation finding. By default, Algorithm \ref{alg:multistart} selects the solution with the largest side length, avoiding degenerate ones whenever possible. 
\begin{example}\label{ex:deltoid}
Consider the \emph{deltoid} curve
\begin{equation}
\gamma_\mathrm{d}(s) = 
\begin{bmatrix}
2\cos s + \cos 2s \\
2\sin s - \sin 2s
\end{bmatrix},
\quad s \in [0,2\pi],
\label{eq:deltoid}
\end{equation}
which is a smooth three-cusped hypocycloid, symmetric under $2\pi/3$ rotations. 
We employ Algorithm \ref{alg:multistart}
for different numbers of agents $n \in \{3, 4, 5\}$ on $\gamma_\mathrm{d}$.  
The focused inspection point is chosen as $c_{\text{target}} = (0,\ 0)$. 

\begin{figure}[h!]
    \centering
    \includegraphics[width=0.8\linewidth]{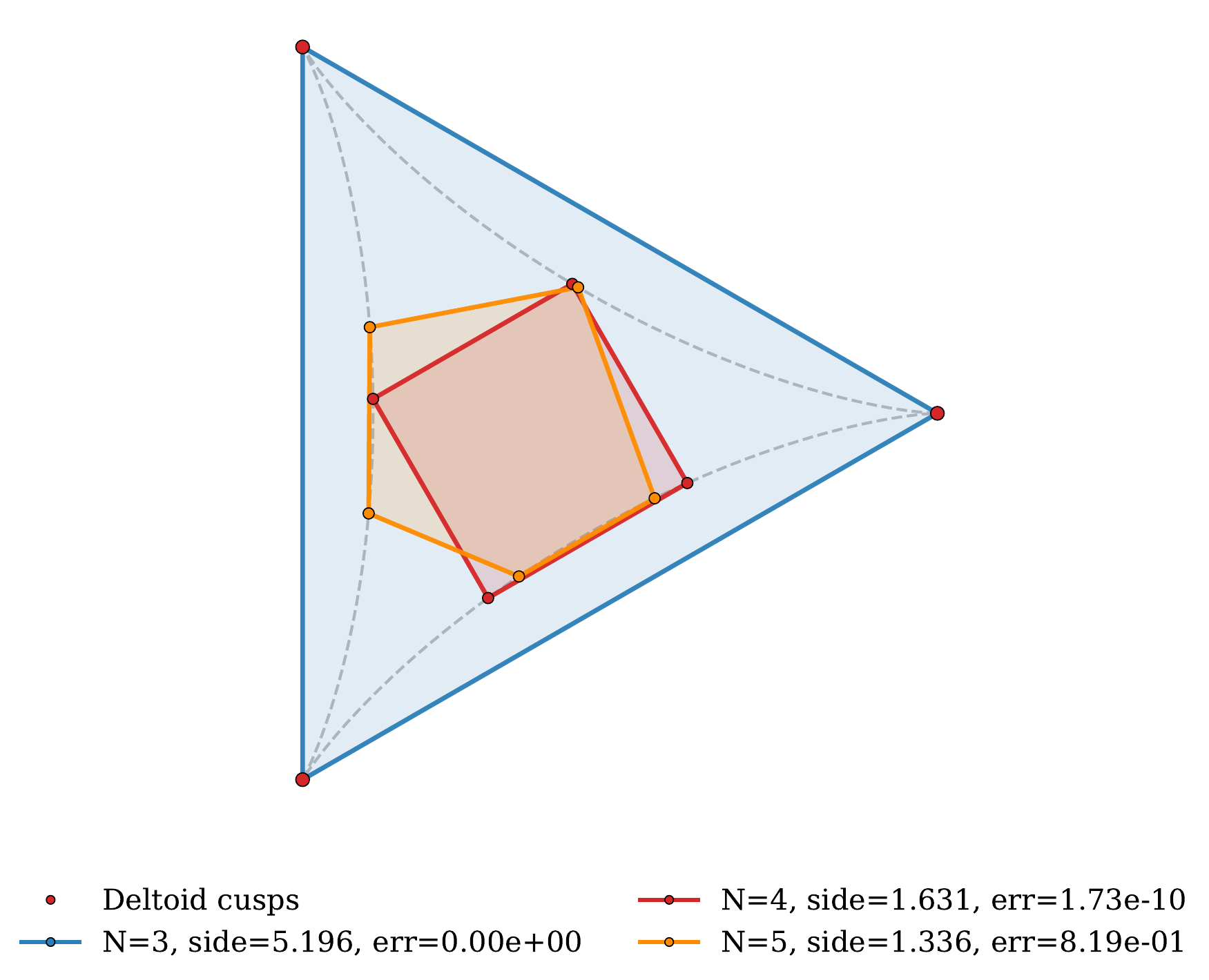}
    \caption{Triangular (blue), square (orange), and pentagonal (green) rigid formations on the deltoid 
    .}
    \label{fig:deltoid-formation}
\end{figure}

\noindent For a triangular rigid formation $n = 3$, starting with a randomized uniformly spaced initialization,
the algorithm converges in a few iterations to the nearest cusps,
yielding an inscribed equilateral triangle.
Due to the threefold rotational symmetry of~\eqref{eq:deltoid}, there are exactly three equivalent minima corresponding to cyclic rotations of $\theta^\star_{3}$.
When $n=4$, the algorithm finds a perfect square configuration (up to numerical thresholds) with a residual $\|r(\theta^\star_{4})\| \approx 1.7\times 10^{-10}$.  
For $n=5$ (pentagon rigid formation), the optimization converges, in few steps to a best-fit configuration with a residual $\|r(\theta^\star_{5})\| \approx 0.8$
\end{example}

The above example highlights how curve symmetry affects the convergence of the algorithm to perfect rigid formations.

\begin{figure*}[t!]
    \centering
    \includegraphics[width=\linewidth]{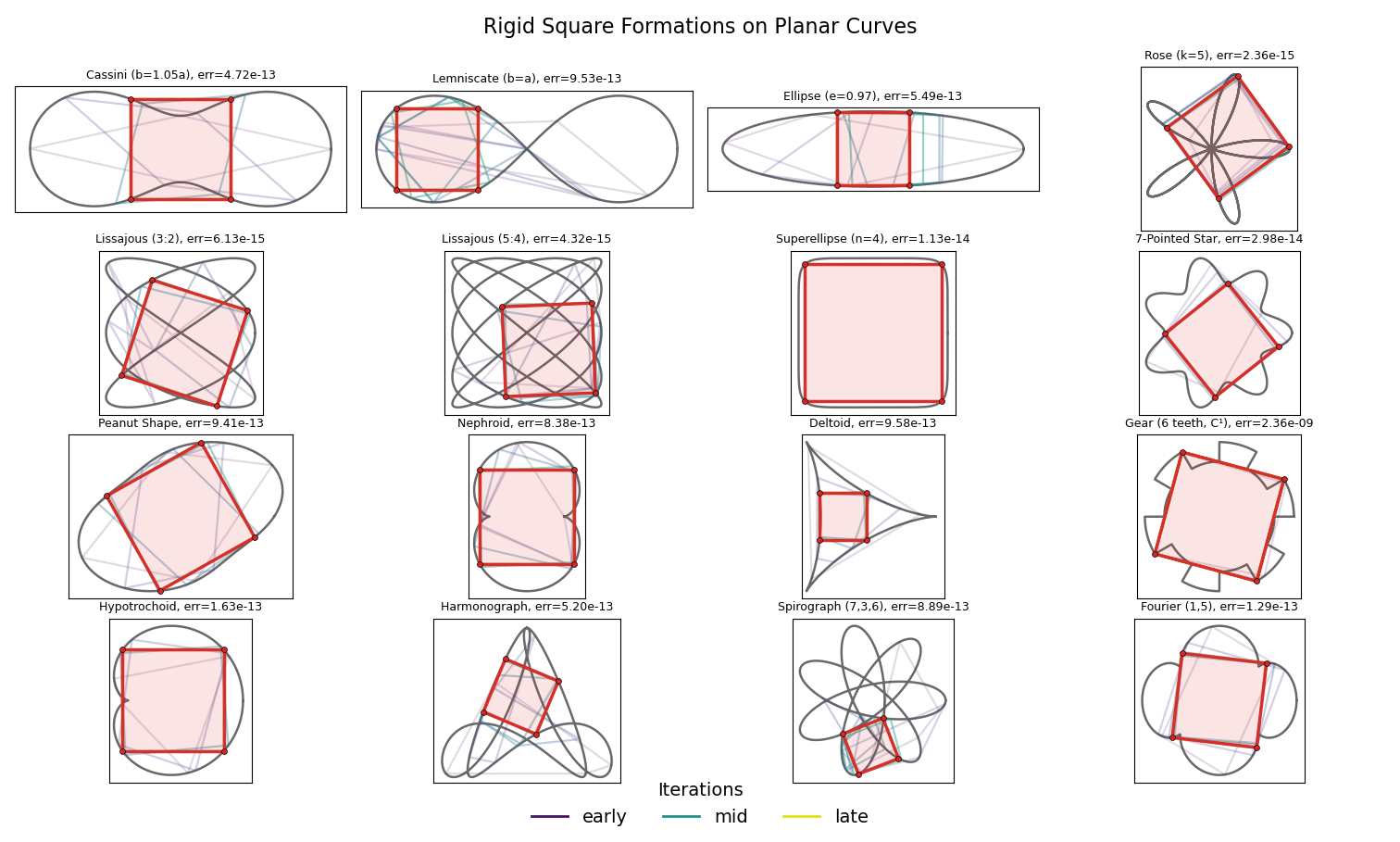}
    \caption{Rigid square formations found on various planar curves. The shaded red quadrilateral indicates the final optimized square, while the softer colored outlines represent intermediate iterations.}
    \label{fig:squares_on_curves}
\end{figure*}

\begin{example}[Inscribed Squares]\label{square-finding}
We illustrate the robustness of Algorithm \ref{alg:multistart} by applying it to a diverse family of planar curves on the famous conjecture due to~\cite{squares}.
Our objective functional (\ref{eq:formation-cost}) is augmented with the following regularizer;
\[
\begin{aligned}
&\|\gamma(\theta_{1})-\gamma(\theta_{3})\| - \sqrt{2}\,\bar{\ell},\quad
 \|\gamma(\theta_{2})-\gamma(\theta_{4})\| - \sqrt{2}\,\bar{\ell},
\end{aligned}
\]
with cyclic indexing and $\bar{\ell}$ the mean side length. We do not provide a target inspection point.
sixteen curves spanning five classes: 
\emph{(i) Smooth and convex} (Ellipse, Superellipse); 
\emph{(ii) Analytic nonconvex} (Cassini ovals, Lemniscates); 
\emph{(iii) Harmonic compositions} (Lissajous, Rose curves, Fourier sums, \enquote{Peanuts}); 
\emph{(iv) Cusped/self-intersecting} (Deltoid, Nephroid, Spirographs); and 
\emph{(v) Piecewise $C^1$} (Gear-like curves with Hermite-interpolation transitions).
Specific parameters are detailed in the provided code. In all instances, the algorithm converges to tight tolerances (e.g., $\|\mathbf{r}\|\le 10^{-9}$).  
Figure~\ref{fig:squares_on_curves} overlays the iterations on each curve; color progression (blue–green–yellow, viridis) encodes early to late iterations, and the final found solution is highlighted.
\end{example}

\subsection{Controller Design}\label{sec:controller-design}
We now turn to the solution to Problem \ref{problem}, namely the design of a smooth feedback law for multiple unicycle-type agents, treating an \textit{augmented} state-space representation and a lifted path manifold.
Each agent must track a reference curve $\Gamma=\gamma([0,2\pi])$, avoid collisions with other agents, and ultimately stabilize at an assigned vertex $\gamma_i^\star$ of a desired rigid formation found by Algorithm \ref{alg:multistart}, given mission objectives. 
The design integrates three ingredients:
\begin{enumerate}[label=(\roman*)]
 \item \emph{Lifted geometric representation} removing ambiguity at the singularities of $\gamma(s)$,
 \item \emph{TFL controller} for the lifted path following manifold, 
 \item \emph{Continuous blending maps} that mixes TFL, pose stabilization, and distributed collision avoidance feedback laws
\end{enumerate}

All components are differentiable in the state variables, resulting in continuity of the closed-loop vector field.


\subsubsection{Lifted Dynamics and TFL path following}

To extend TFL path-following to non-simple curves, we \emph{lift} the agent’s state by introducing a virtual coordinate $(z_i, v_{z,i})$ governed by simple integrators,
\begin{equation}\label{eq:lifted-virtual}
\dot{z}_i = v_{z,i}, \qquad \dot{v}_{z,i} = a_{z,i},
\end{equation}
where $a_{z,i}$ is an auxiliary control input. The augmented state and input are
\[
\mathbf{x}_L = [x, y, \psi, v, z, v_z]^\top, \qquad 
\mathbf{u} = [a, \omega, a_z]^\top,
\]
yielding
\[
\dot{\mathbf{x}}_L = \mathbf{f}(\mathbf{x}_L) + 
\mathbf{g}_1(\mathbf{x}_L)a + 
\mathbf{g}_2(\mathbf{x}_L)\omega + 
\mathbf{g}_3(\mathbf{x}_L)a_z,
\]
with
\begin{align*}
\mathbf{f} &= [v\cos\psi,\, v\sin\psi,\, 0,\, 0,\, v_z,\, 0]^\top, \\
\mathbf{g}_1 &= [0,0,0,1,0,0]^\top, \\
\mathbf{g}_2 &= [0,0,1,0,0,0]^\top, \quad
\mathbf{g}_3 = [0,0,0,0,0,1]^\top.
\end{align*}
The virtual state $z_i$ defines a dynamic path parameter
\[
s_i(t) = \frac{z_i(t)}{K_L}, \qquad K_L > 0,
\]
which serves as a smooth address along the desired curve $\gamma(s_i)$. The original 2D curve $\Gamma$ is thus embedded \textit{unwrapped} in $\mathbb{R}^3$ as
\begin{equation}\label{eq:lifted-embedding}
\phi(s) =
\begin{bmatrix}
\gamma_x(s)\\[2pt]
\gamma_y(s)\\[2pt]
K_L s
\end{bmatrix}, \qquad K_L>0,
\end{equation}
which is by construction simple and free of self-intersections.

One is then able to apply TFL to stabilize the augmented agent dynamics on the manifold defined by \eqref{eq:lifted-embedding}. The transverse and tangential outputs are defined as ($\mathbf{n}(s_i), \mathbf{t}(s_i)$ are the Frenet nomal and tangent and $\psi_t$ is the tangent orientation angle)
\begin{align}
\mathbf{y}_{\perp} &=
\begin{bmatrix}
h_1 \\ h_2
\end{bmatrix}
=
\begin{bmatrix}
\mathbf{n}(s_i)^\top(\mathbf{p}_i-\gamma(s_i))\\
\mathbf{t}(s_i)^\top(\mathbf{p}_i-\gamma(s_i))
\end{bmatrix}
=\begin{bmatrix}e_n\\ e_t\end{bmatrix}, \label{eq:output-perp}\\[3pt]
y_{\parallel} &= h_3 = z_i - z_i^\star. \label{eq:output-parallel}
\end{align}
Since $s_i$ evolves from $z_i$ rather than through orthogonal projection, $e_{t,i}$ remains a valid dynamic variable. 
The second derivatives yield (with the standard notation for the Lie derivative $L_{\mathbf{f}}^2 h = L_{\mathbf{f}}(L_{\mathbf{f}} h), \ L_{\mathbf{f}} h(q) = \frac{\partial h}{\partial q} \mathbf{f}(q)$)
\[
\ddot{\mathbf{y}} = L_{\mathbf{f}}^2\mathbf{h} + \mathbf{D}(\mathbf{x}_L)\mathbf{u},
\]
with (denoting $\kappa$ for the curvature of $\gamma$)
\[
\mathbf{D}(\mathbf{x}_L)=
\begin{bmatrix}
\sin(\psi-\psi_t) & v\cos(\psi-\psi_t) & -\frac{\kappa\|\gamma'\|e_t}{K_L}\\
\cos(\psi-\psi_t) & -v\sin(\psi-\psi_t) & \frac{\kappa\|\gamma'\|e_n-\|\gamma'\|}{K_L}\\
0 & 0 & 1
\end{bmatrix},
\]
whose determinant $\det(\mathbf{D})=-v$. Hence, the system has vector relative degree $\mathbf{r}=(2,2,2)$ and total relative degree $r=n=6$, i.e., it is fully feedback linearizable.

The path-following manifold
\[
\mathcal{Z} = \{\mathbf{x}_L\mid \mathbf{y}_{\perp}=\dot{\mathbf{y}}_{\perp}=\mathbf{0}\}
\]
is a 2D invariant submanifold of $\mathbb{R}^6$. Define the coordinates
\begin{align*}
\mathbf{\xi}_{\perp} &= [h_1, L_{\mathbf{f}}h_1, h_2, L_{\mathbf{f}}h_2]^\top = [e_n,\dot{e}_n,e_t,\dot{e}_t]^\top,\\
\mathbf{\xi}_{\parallel} &= [h_3, L_{\mathbf{f}}h_3]^\top = [z-z^\star, v_z-\dot{z}^\star]^\top,
\end{align*}
the linearized dynamics read
\[
\ddot{\mathbf{y}}_{\perp} = \mathbf{v}_{\perp}, \qquad \ddot{y}_{\parallel} = v_{\parallel},
\]
where $\mathbf{v} = [\mathbf{v}_{\perp}^\top, v_{\parallel}]^\top$ is a virtual input defined by
\begin{align}\label{eq:tfl-feedback}
\mathbf{u}_{\text{TFL}} = \mathbf{D}^{-1}(\mathbf{x}_L)\big(\mathbf{v} - L_{\mathbf{f}}^2\mathbf{h}(\mathbf{x}_L)\big).
\end{align}
A simple PD choice $\mathbf{v} = -K_p\mathbf{y} - K_d\dot{\mathbf{y}}$ ensures exponential convergence $\mathbf{y}\to0$, thereby enforcing attraction to $\mathcal{Z}$ and along-curve progression with $z^\star(t)$.

\begin{example}
Consider the 3-petals rose with the following parametrization
\begin{equation}
\gamma(s) =
\begin{bmatrix}
a \cos(ks) \cos(s) \\
a \cos(ks) \sin(s)
\end{bmatrix},
\quad a=1.8, \, k=3.
\label{eq:rose}
\end{equation}
For agent $i$, one needs to evaluate:
\begin{align*}
    \gamma'(s)
&= 
1.8
\begin{bmatrix}
-\sin(s)\cos(3s) - 3\sin(3s)\cos(s)\\[3pt]
\cos(s)\cos(3s) - 3\sin(3s)\sin(s)
\end{bmatrix}, \\[4pt]
\mathbf{t}(s) &= 
\frac{\gamma'(s)}{\|\gamma'(s)\|}, \\[4pt]
\mathbf{n}(s) &= 
\frac{1}{\|\gamma'(s)\|}
\begin{bmatrix}
-\big(\cos(s)\cos(3s) - 3\sin(3s)\sin(s)\big)\\[3pt]
-\big(\sin(s)\cos(3s) + 3\sin(3s)\cos(s)\big)
\end{bmatrix}, \\[4pt]
\kappa(s)
&=
\frac{|\,8\cos(3s)\sin(3s) + 3\cos(2s)\cos(6s)\,|}
     {1.8\big(\cos^2(3s) + 9\sin^2(3s)\big)^{3/2}}.
\end{align*}
\end{example}
\subsubsection{Pose Stabilization for Targeted Inspection}

Let $\gamma_i^\star\ = \gamma(\theta_i^\star)$ denote the assigned vertex of the target formation and $\psi_i^\star$ a desired orientation. 
In a neighborhood of $(\gamma_i^\star,\psi_i^\star)$, a static pose–stabilizing control is applied:
\begin{align}\label{eq:pose-reg}
    \mathbf{u}_{i,\mathrm{pose}}  = 
\begin{bmatrix}
-k_v v_i - k_p\big((\mathbf{p}_i-\gamma_i^\star)^\top \mathbf{h}_i\big)\\[3pt]
-k_\psi(\psi_i-\psi_i^\star)\\[3pt]
-k_z v_{z,i}
\end{bmatrix},
\end{align}

where $\mathbf{h}_i = [\cos\psi_i, \sin\psi_i]^\top$. 
This feedback law exponentially damps the velocities, driving the agent’s pose toward $(\gamma_i^\star,\psi_i^\star)$.

To achieve a seamless transition between path following and pose stabilization, a continuously differentiable blending map
\begin{equation}\label{eq:sigma-blend}
\sigma_i(\mathrm{revs}_i, d_i) = 
\beta\!\left( \frac{\mathrm{revs}_i}{\mathrm{revs}^\star} \right)
\!\cdot\!
\left( 1 - \beta\!\left( \frac{d_i}{d_{\mathrm{sw}}} \right) \right),
\end{equation}
is defined as a function of the number of completed path revolutions, $\mathrm{revs}_i$, and the Euclidean distance to the target vertex, $d_i = \|\mathbf{p}_i - \gamma_i^\star\|$. 
Here, $\mathrm{revs}^\star$ and $d_{\mathrm{sw}}$ denote, respectively, the desired number of sweeping revolutions and a switching radius. 
The sigmoid $\beta(\xi) = 3\xi^2 - 2\xi^3$, $\xi \in [0,1]$, provides smooth activation.
The first factor in~\eqref{eq:sigma-blend} corresponds to sweeping mission completion, while the second  strengthens the pose controller near $\gamma_i^\star$. 

The resulting \textit{nominal} agent control is a convex combination of the path–following and pose–stabilizing inputs:
\begin{equation}\label{eq:nominal-blend}
\mathbf{u}_{i,\mathrm{nom}} = (1-\sigma_i)\,\mathbf{u}_{i,\mathrm{TFL}} + \sigma_i\,\mathbf{u}_{i,\mathrm{pose}}.
\end{equation}

\subsubsection{Distributed Collision Avoidance}

Let $\mathcal{N}_i$ denote the set of neighboring agents within a sensing radius of agent $i$. 
Each neighbor $j\in\mathcal{N}_i$ exerts a repulsive potential
\begin{equation}
\psi_{ij} = k_{\mathrm{avoid}}\!\left(\frac{1}{r_{ij}} - \frac{1}{d_{\mathrm{ao}}}\right)\!\frac{1}{r_{ij}^2}, 
\qquad r_{ij} = \|\mathbf{p}_i - \mathbf{p}_j\| < d_{\mathrm{ao}},
\end{equation}
where $d_{\mathrm{ao}}$ is the activation range and $k_{\mathrm{avoid}} > 0$. 
The total repulsive force acting on agent $i$ reads
\begin{align}
\mathbf{F}_i^{\mathrm{rep}} &= \alpha_i^{\mathrm{duty}} 
\sum_{j\in\mathcal{N}_i} \psi_{ij} (\mathbf{p}_i - \mathbf{p}_j),\\
\alpha_i^{\mathrm{duty}} &= 
\beta\!\left(\frac{\sigma_{\mathrm{accept}}-\sigma_i}{\delta_{\sigma}}\right), \label{eq:duty}
\end{align}
where $\alpha_i^{\mathrm{duty}}$ gradually deactivates repulsion as $\sigma_i\!\to\!1$, preventing agents already settled at their targets from reacting to others.

The corresponding avoidance feedback law along $\mathbf{F}_i^{\mathrm{rep}}$:
\begin{equation}\label{eq:avoidance}
\mathbf{u}_{i,\mathrm{avoid}} = 
\begin{bmatrix}
k_{v,\mathrm{a}} (v_{\mathrm{des}} - v_i) \\[3pt]
k_{\omega,\mathrm{a}} (\psi_{\mathrm{des}} - \psi_i) \\[3pt]
-k_{z,\mathrm{a}} v_{z,i}
\end{bmatrix},
\end{equation}
with desired direction and speed
\[
\psi_{\mathrm{des}} = \mathrm{atan2}((\mathbf{F}_i^{\mathrm{rep}})_y, (\mathbf{F}_i^{\mathrm{rep}})_x), 
\, \,
v_{\mathrm{des}} = v_{\mathrm{max},i}\cos(\psi_{\mathrm{des}} - \psi_i).
\]

A blending factor $\alpha_i$ regulates the influence of avoidance versus nominal control:
\begin{equation}\label{eq:alpha-scale}
\alpha_i = \max_{j\in\mathcal{N}_i} 
\left\{ \alpha_i^{\mathrm{duty}} \cdot 
\beta\!\left(\frac{d_{\mathrm{ao}}-r_{ij}}{d_{\mathrm{ao}}-d_{\mathrm{safe}}}\right)
\right\},
\end{equation}
where $d_{\mathrm{safe}}$ is the minimum allowable separation. 
The final distributed control law is therefore
\begin{equation}\label{eq:final-control}
\mathbf{u}_i = (1-\alpha_i)\,\mathbf{u}_{i,\mathrm{nom}} + \alpha_i\,\mathbf{u}_{i,\mathrm{avoid}}.
\end{equation}

\begin{remark}
Equations~\eqref{eq:avoidance}–\eqref{eq:final-control} depend only on local neighbor states, ensuring decentralized implementation. 
When two agents approach, the one with smaller $\sigma_i$ (farther from its vertex) yields, since its $\alpha_i$ remains positive. 
Agents already stabilized at their vertices are unaffected, preserving the desired formation geometry.
\end{remark}


Some comments are in order. When $\alpha_i = 0$, the nominal controller~\eqref{eq:nominal-blend}
smoothly blends the TFL law
$\mathbf{u}_{i,\mathrm{TFL}}$ with the pose regulator
$\mathbf{u}_{i,\mathrm{pose}}$.  
During the sweeping phase, the TFL component renders the
path-following manifold
$\mathcal{M}_i = \{(\mathbf{p}_i, v_i, z_i, v_{z,i}) :
\mathbf{y}_{i,\perp} = \dot{\mathbf{y}}_{i,\perp} = 0\}$
invariant and attractive, and the motion of each agent follows the
lifted dynamics along the desired path.  
As the number of completed revolutions approaches the target and the
distance to the assigned vertex $\gamma(\theta_i^\star)$ becomes small,
the blending map gradually activates the pose–stabilization dynamics.
This smooth transition relaxes path invariance without discontinuity:
although $\mathbf{u}_{i,\mathrm{pose}}$ may drive the agent slightly away
from the curve, its PD structure ensures bounded deviation and motion
in the vicinity of $\Gamma$, ensuring \textit{practical invariance}.

When $\alpha_i > 0$, the avoidance term $\mathbf{u}_{i,\mathrm{avoid}}$ generates a repulsive vector field that ensures pairwise distances satisfy $r_{ij}(t) \ge d_{\mathrm{safe}}$ for all $t \ge 0$, provided $r_{ij}(0) > d_{\mathrm{safe}}$.  
The asymmetric duty factor $\alpha_i^{\mathrm{duty}}$ guarantees that as agent $i$ approaches its target $\gamma(\theta_i^\star)$, $\sigma_i \to 1$ and thus $\alpha_i \to 0$, ensuring that avoidance deactivates smoothly and the nominal control becomes dominant again.  
This design precludes spurious avoidance equilibria and preserves convergence toward the desired formation.

\begin{proposition}
Given an admissible curve $\gamma(s)$ and a desired rigid formation 
$P(\theta^\star) \subset \Gamma$, 
assume each agent $i$ is assigned a unique target $\gamma(\theta_i^\star)$ and that the formation is non-degenerate, i.e., 
$\min_{i \neq j} \|\gamma(\theta_i^\star) - \gamma(\theta_j^\star)\| > d_{\mathrm{ao}}$.  
Then, there exist controller gains and blending parameters in~\eqref{eq:final-control} solving Problem \ref{problem}
\end{proposition}
\begin{proof}[Sketch]
Let $V_{\perp}, V_{\mathrm{pose}}$ be Lyapunov functions for the TFL and pose subsystems, satisfying $\dot{V}_k \le -c_k V_k$. These exists by construction. Consider the candidate $V_i=(1-\sigma_i)V_{\perp}+\sigma_i V_{\mathrm{pose}}$. Its time derivative is
\[
\dot V_i=(1-\sigma_i)\dot V_{\perp}+\sigma_i\dot V_{\mathrm{pose}}
      +\dot\sigma_i\,(V_{\mathrm{pose}}-V_{\perp}).
\]
Differentiating (20) yields $\dot{\sigma}_i = \frac{\partial \sigma_i}{\partial \text{revs}}\dot{\text{revs}}_i + \frac{\partial \sigma_i}{\partial d_i} \dot{d}_i$. Since $\dot{\text{revs}}_i \propto v_i$ and $\dot{d}_i \le \|\dot{\mathbf{p}}_i\|$, there exists $L>0$ such that $|\dot{\sigma}_i| \le L \|\dot{\mathbf{x}}_i\|$.
By choosing $\text{revs}^\star$ large, $V_{\perp} \to 0$ before $\sigma_i$ activates. Thus, the negative definite terms $(1-\sigma_i)\dot V_{\perp} + \sigma_i\dot V_{\mathrm{pose}}$ dominate the perturbation $\dot{\sigma}_i V_{\mathrm{pose}}$ near the target.

Regarding collision avoidance, the control is $\mathbf{u}_i = \mathbf{u}_{\text{nom}} + \alpha_i \mathbf{u}_{\text{avoid}}$. As agent $i$ converges to $\gamma(\theta_i^\star)$, $d_i \to 0$ implies $\sigma_i \to 1$.
Specifically, when $\sigma_i(t) > \sigma_{\text{accept}}$ the numerator in (\ref{eq:duty}) becomes negative: $(\sigma_{accept} - \sigma_i)/\delta_\sigma < 0$.
By definition of $\beta(\xi)$, $\xi \le 0 \implies \beta(\xi)=0$.
Consequently, $\alpha_i^{\text{duty}} = 0$ identically in a neighborhood of the target. The avoidance term vanishes ($\alpha_i = 0$), and the system recovers the asymptotic stability of the nominal controller.
\end{proof}

\section{Simulations}
A team of four agents ($N=4$) is tasked with sweeping, then forming a rigid square on two geometrically challenging planar curves: the \emph{Deltoid}, which features three cusps, and the \emph{Lissajous curve} (3:2 frequency ratio) with a self-intersection. We assume, after sufficient sweeping that a target of inspection is found. In the deltoid case we set $c_{\text{target}}$ to the origin, while in the Lissajous case, a slightly off-center inspection target is given $c_{\text{target}} = (0.5, \ 0.5)$. The simulations demonstrate the performance of the proposed controller in a typical sweep and focused-inspection scenario. For each curve, the agents determine the desired rigid formation via Algorithm~\ref{alg:multistart}. Controller parameters are set as follows: 
$(K_p, K_d) = (\mathrm{diag}(15,15,3), \mathrm{diag}(10,10,3))$ for the lifted TFL law~\eqref{eq:tfl-feedback}, and $(k_p, k_v, k_\psi)=(3.0, 4.0, 5.0)$ for the pose regulator~\eqref{eq:pose-reg}. 
The blending function~\eqref{eq:sigma-blend} uses $\mathrm{revs^\star}=1.0$ revolutions and $d_{\mathrm{sw}}=15\%$ of the curve's characteristic scale. 
Collision avoidance activates at distance $d_{\mathrm{avoid}}=12\%$ of the curve scale, with a hard safety radius $d_{\mathrm{safe}}=6\%$. 
Avoidance gain is set to $\kappa_{\mathrm{avoid}}=0.8$ with directional modulation reducing repulsion to $30\%$ for co-directional motion.
Initial positions are randomly distributed near the curve with small angular perturbations from a fixed seed.

Figure~\ref{fig:sim_snapshots} shows representative trajectories at three key phases.  
On the Deltoid (top row), agents are initially scattered near their random starting configurations (a), converge to the curve and begin coordinated path following by $t=20$s (b), and achieve the final square formation by $t=100$s (c), with all $\sigma_i\!\approx\!1$ indicating full pose stabilization.  
The Lissajous curve (bottom row) demonstrates similar convergence behavior despite the non-defult targeted inspection center. Notably, agents successfully navigate through the curve's self-intersections without collision, validating the directional avoidance mechanism.

\begin{figure*}[t!]
    \centering
    \includegraphics[width=0.8\linewidth]{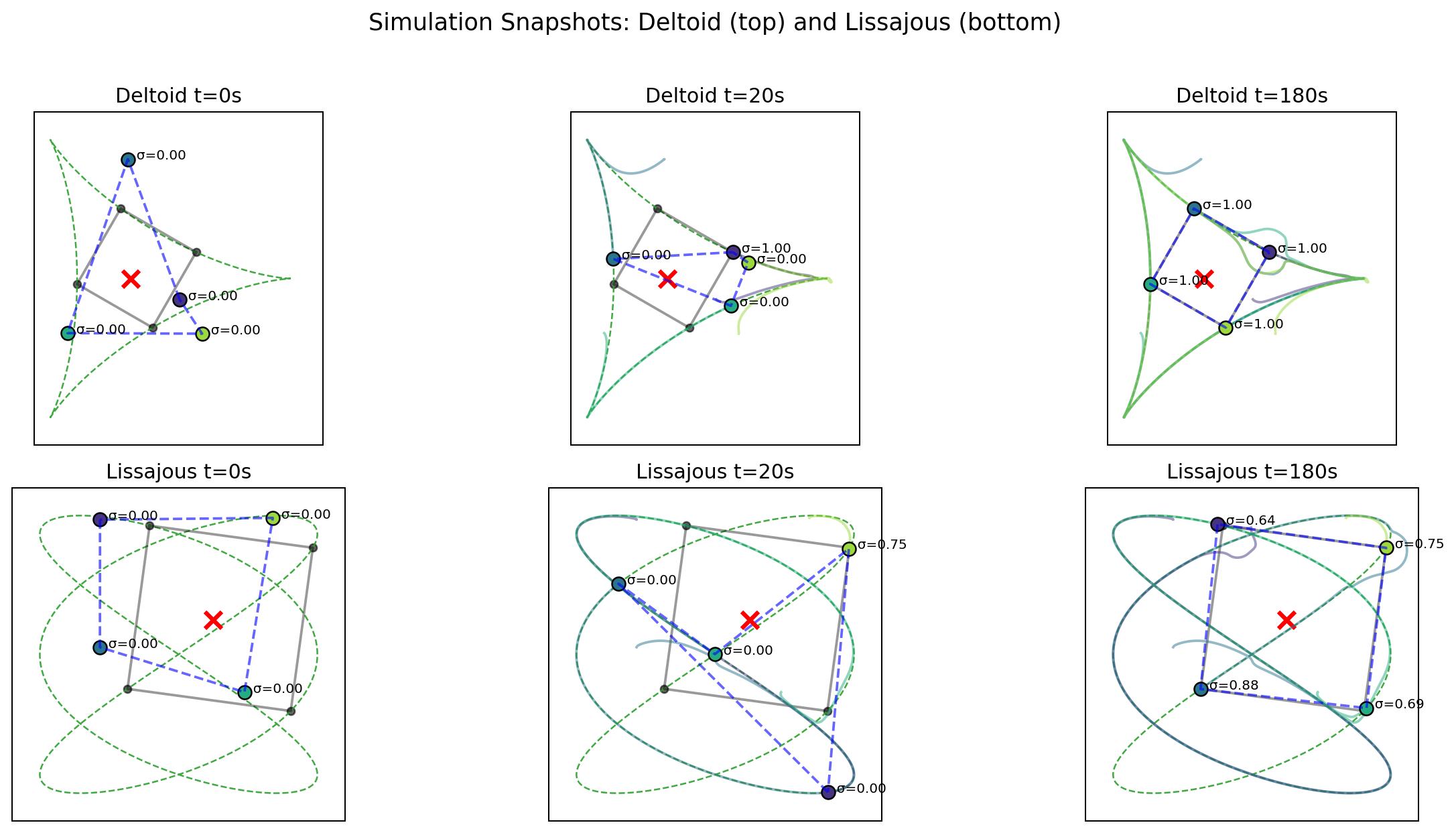}
    \caption{Snapshots. Each panel shows agent positions (colored dots), historical trajectories (fading trails), target formation vertices (black circles), and blending parameters $\sigma_i$. The  final inspection target is shown in red.}
    \label{fig:sim_snapshots}
\end{figure*}

Fig.~\ref{fig:sim_plots} shows performances on the Lissajous curve. 
Panel (a) shows the minimum inter-agent distance throughout the simulation. While the distance remains above $d_{\mathrm{safe}}$ for the majority of the mission, with falling to the minimum safe distance for brief transients, while (b) curve adherence (mean projection error) remains small showcasing nominal invariance apart from when avoidance or complete pose regulation transitions occur. Individual agent traces (faint lines) reveal asynchronous convergence depending on initial conditions and avoidance triggers.  

\begin{figure}[!h]
    \centering
    \includegraphics[width=\linewidth]{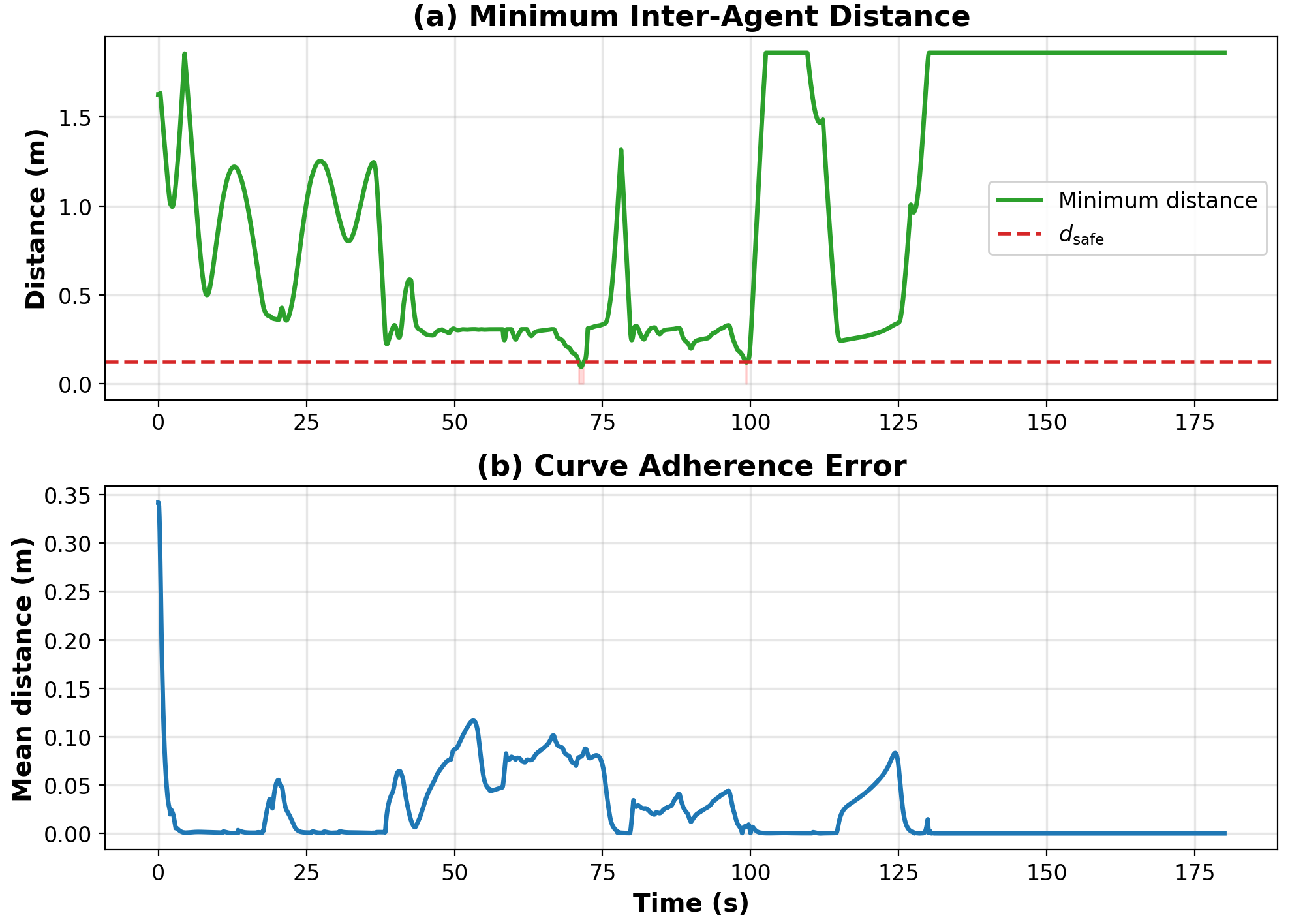}
    \caption{Performance metrics in the Lissajous case. (a) Minimum inter-agent distance
with safety threshold. (b) Curve adherence error during path-following phase.}
    \label{fig:sim_plots}
\end{figure}

Several implementation aspects warrant comments. The blending function~\eqref{eq:sigma-blend} interpolates between product (enforcing both revolution and distance constraints) and minimum (preventing deadlock) based on constraint divergence. This resolves the deadlock issue where agents completing sufficient revolutions but displaced by avoidance maneuvers would remain stuck in path-following mode.  Additionally, an \emph{Asymmetric avoidance} mechanism is included. Agents with $\sigma_i\!>\!0.85$ reduce their avoidance radius from $d_{\mathrm{avoid}}$ to $1.5 d_{\mathrm{safe}}$, creating a smaller threshold while not blocking approach paths for late-arriving agents. This asymmetric design preserves formation stability while enabling sequential convergence.  
To facilitate this, the repulsion force is modulated by heading difference, with stricter force for head-on encounters. This prevents agents from impeding each other's progress while maintaining collision safety.  All these aspects and further simulations can be found in the supplementary material.

\section{CONCLUDING REMARKS}

We presented a unified framework for path following and stabilization of rigid formations
on general planar~$C^1$ curves, including those with cusps and
self–intersections. A randomized multi–start Gauss–Newton algorithm
was introduced to compute admissible inscribed polygonal
formations, while a lifted transverse–feedback–linearization
controller was developed to ensure smooth convergence to a desired
path. This controller is continuously blended with a pose–stabilizing
mechanism, enabling convergence to assigned formation vertices while
preserving inter–agent safety through distributed avoidance.

Future, and currently ongoing work will focus on experimental validation using wheeled
robots, and on extending the proposed framework to spatial curves
in~$\mathbb{R}^3$ through experimental studies with aerial vehicles.

\balance
\bibliography{ifacconf}

\end{document}